\documentclass[letterpaper, 10 pt, conference]{ieeeconf}  

\IEEEoverridecommandlockouts                              

\usepackage{graphics} 
\usepackage{epsfig} 
\usepackage{mathptmx} 
\usepackage{times} 
\usepackage{amsmath} 
\usepackage{amssymb}  
\usepackage{bbm}
\usepackage{theorem} 
\usepackage{algpseudocode}
\usepackage{stmaryrd}
\usepackage[font=small]{caption}
\usepackage{subcaption}
\usepackage{url}
\usepackage[linesnumbered,ruled,vlined]{algorithm2e}
\usepackage{booktabs}
\usepackage{multirow}

\newcommand{\Xc}{{\mathcal{X}}}

\newtheorem{prop}{Proposition}

\newtheorem{definition}{Definition}

\newtheorem{lemma}{Lemma}

\usepackage[dvipsnames]{xcolor}

\title{\LARGE \bf 
Online Distribution Shift Detection via Recency Prediction
}

\author{\\
\normalsize Rachel Luo, Rohan Sinha, Yixiao Sun, Ali Hindy, Shengjia Zhao, Silvio Savarese,\\
\normalsize Edward Schmerling, Marco Pavone
\thanks{R. Luo, R. Sinha, Y. Sun, A. Hindy, S. Zhao, S. Savarese, E. Schmerling, and M. Pavone are with Stanford University, Stanford, CA, USA; \{rsluo, rhnsinha, alvinsun, ahindy, sjzhao, ssilvio, schmrlng, pavone\}@stanford.edu.}
\noindent \thanks{The NASA University Leadership Initiative (grant \#80NSSC20M0163) provided funds to assist the authors with their research, but this article solely reflects the opinions and conclusions of its authors and not any NASA entity.}
}
\date{}

\begin{document}

\maketitle
\thispagestyle{empty}
\pagestyle{empty}

\begin{abstract}

    When deploying modern machine learning-enabled robotic systems in high-stakes applications, detecting distribution shift is critical. However, most existing methods for detecting distribution shift are not well-suited to robotics settings, where data often arrives in a streaming fashion and may be very high-dimensional. In this work, we present an online method for detecting distribution shift with guarantees on the false positive rate --- i.e., when there is no distribution shift, our system is very unlikely (with probability $< \epsilon$) to falsely issue an alert; any alerts that are issued should therefore be heeded. Our method is specifically designed for efficient detection even with high dimensional data, and it empirically achieves up to 11x faster detection on realistic robotics settings compared to prior work while maintaining a low false negative rate in practice (whenever there is a distribution shift in our experiments, our method indeed emits an alert). We demonstrate our approach in both simulation and hardware for a visual servoing task, and show that our method indeed issues an alert before a failure occurs.

\end{abstract}

\section{Introduction}

Machine learning (ML) models deployed in the real world often encounter test time inputs that do not follow the same distribution as the training time inputs because autonomous robots continuously encounter new situations when deployed; in other words, there is \textit{distribution shift} (also known as domain shift). However, standard machine learning practice operates under the assumption that the training and test distributions are identical, and thus learned models may not perform well under changed conditions. Consequently, methods that detect distribution shifts are necessary to maintain the reliability of modern ML-enabled systems, especially in high-stakes situations such as aircraft control, autonomous driving, or medical decision-making. However, most existing methods for detecting a distribution shift operate only in an offline batch setting, whereas in robotics, detecting distribution shifts in an online manner is particularly important: knowledge of gradually shifting distributions throughout continuous long-term deployment cycles can trigger safety-preserving interventions and subsequent model refinement or retraining.

\begin{figure*}[t]
    \centering
    \includegraphics[width=\linewidth]{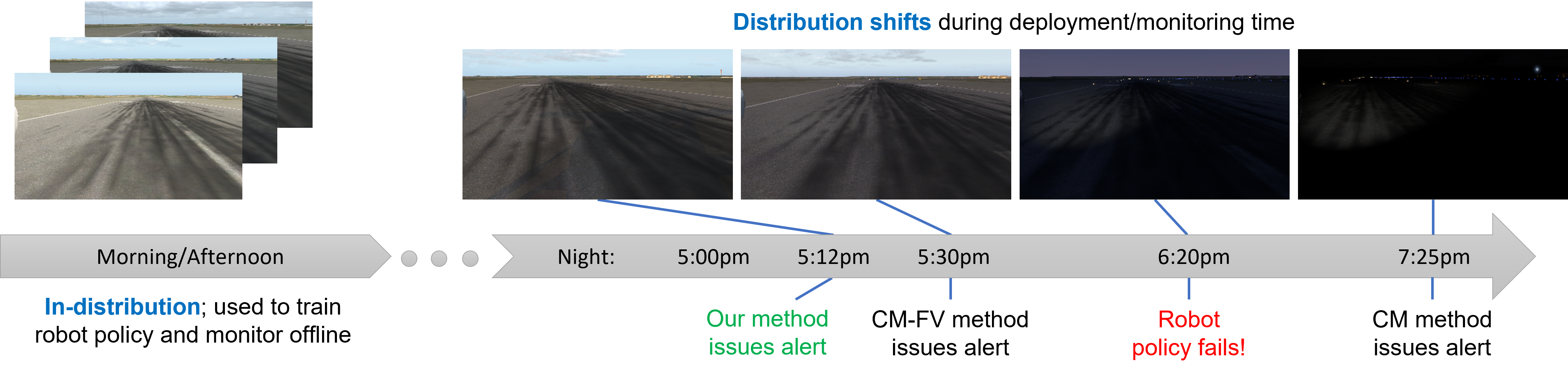}
    \caption{Illustration of our problem setting and high-level approach. Learning enabled robotics systems, such as a vision-based aircraft controller, are trained on data from a finite set of environments (e.g. images taken in the morning and afternoon). When deployed, these systems may operate in distribution-shifted conditions, resulting in erroneous predictions on out-of-distribution data. To improve safety, we design a warning system that can detect distribution shifts in a streaming fashion with a \textit{guaranteed} false positive rate. 
    }
    \label{fig:system}
\end{figure*}

Therefore, in this work, we consider the problem of detecting distribution shifts online when conditions shift gradually across episodes. In such cases, the ideal warning system satisfies three desiderata: 1) it quickly issues an alert \textit{before} undesirable or dangerous situations arise due to the magnitude of the distribution shift, 2) it has valid performance guarantees in an online setting, and 3) it achieves a low false positive rate, as any system that gives too many extraneous warnings will not be useful in practice. Accordingly, our framework can lead to desirable safety or performance outcomes in applications with a strong notion of task repetition. For example, consider (a fleet of) autonomous aircraft repeatedly taking off from a set of runways during a continuous deployment so that each taxiing sequence constitutes an episode drawn from a task distribution. The planes' sensors may degrade over time, or lighting conditions may change significantly over the course of the day, causing operational conditions to drift away from the training regime. Our method may be used to alert to this shift before performance degrades significantly and to prompt a manufacturer to improve the robot's software for these shifted conditions. 
Note that the nature of the distribution shift detection problem is different from that of anomaly detection (e.g. \cite{CPforML}, \cite{Xu2021ConformalAD}, \cite{Kaur2023CODiT}); it is arguably impossible to make a distributional claim without multiple samples of evidence to reason about whether the distribution has shifted or whether the system is simply experiencing a rare event.
Therefore, our algorithm is not intended to, e.g., trigger safety interventions within an episode in response to a sudden rare or ``out-of-distribution'' event. Instead, it should inform decision-making about successive deployments and guide iterative development of system components \cite{SinhaSharmaEtAl2022}. 

Our algorithm is designed to detect distribution shifts as quickly as possible; therefore, it allows system designers to \textit{proactively} detect gradually shifting conditions before they lead to failure. This is in contrast to a monitor on an output metric (e.g., airplane takeoff success rate, package grasp success rate, classification accuracy, etc.), which may issue a warning only \textit{after} the system performance has already degraded. If a task performance metric is the only feature considered, one could thus be lulled into a false sense of security before performance suddenly drops catastrophically in situations where the input distribution shift is gradual, but the task performance is discontinuous in shift; for example, a bolt could slowly loosen (but remain serviceable) until it suddenly falls out. 

Additionally, our approach is practically useful because well-engineered systems are often equipped with methods to catch or compensate for various types of errors. This means that these errors may remain hidden until a buildup occurs such that the system can no longer compensate (for instance, simultaneous distribution shifts for multiple components). Thus, monitoring for distribution shifts in individual components (in particular, for insidious gradual distribution shifts) can be very valuable.

\textbf{Contributions:} 1) We present a method that quickly alerts users when a distribution shift has occurred while also providing a guaranteed (low) false positive rate below a user-defined risk tolerance, in an online setting. 
When there is no distribution shift, no warning will be issued with at least $1-\epsilon$ probability, so any emitted warnings should be heeded. 
Our method is applicable to general time series data; the input could take any form, including images, system dynamics, or any other set of features. Specifically, we target gradual distribution shifts over episodic situations. 
2) Our approach improves upon existing methods, which are either offline or not tailored for high-dimensional inputs, because we directly train a neural network model to predict whether a new sample at test time differs meaningfully from previous samples and construct a martingale from the model outputs to issue warnings with guarantees. 3) Because we construct a warning signal that grows exponentially under distribution shift, our approach empirically allows us to detect online shifts more rapidly than existing approaches. 4) We empirically evaluate our approach on photorealistic simulations of an autonomous aircraft taxiing down a runway using a camera perception module in the X-Plane 11 flight simulator, and in hardware on a free-flyer space robotics testbed for vision-based navigation. Our approach detects gradually degrading conditions up to eleven times as rapidly as a baseline while maintaining a guaranteed false positive rate of 1\%. Moreover, our proposed method quickly detects a distribution shift before a failure occurs when there is one across all our experiments, demonstrating that our method performs well on realistic examples. We conclude that our method is attractive for robotic applications as it is practical and tailored to detect shifts in the high-dimensional and sequentially observed inputs of ML models, like perception systems, used in a robotic autonomy stack.

\section{Related Work}
\label{sec:related}

Machine learning models can perform poorly and erratically on test data drawn from a distribution that differs from the training distribution. Mitigating the impact of distribution shifts is a long-standing challenge, and empirical studies show that subtle shifts still severely impact the performance of state-of-the-art models (e.g., see \cite{hendrycks2019Benchmarking, GeirhosJacobsenEtAl2020, KohSagawaEtAl2021, MillerTaoriEtAl2021}). Furthermore, as learned components find increasing use in robotics stacks, erroneous predictions induced by distribution shifts can cause dangerous system-level failures in safety-critical applications. For example, a robot using a classification model trained on daytime data can have accidents when deployed at night. Therefore, we must develop methods to detect distribution shifts to avoid system failures in shifted conditions.

Traditional approaches for detecting distribution shift use statistical hypothesis testing to determine whether the test-time distribution differs from the training distribution (\cite{GrettonBorgwardtEtAl2012, Rabanser2019FailingLA, Kulinski2020FeatureSD, Kamulete2021TestFN}). For example, \cite{GrettonBorgwardtEtAl2012} develop a hypothesis test based on evaluating the maximum mean discrepancy (MMD) and similarly, \cite{Rabanser2019FailingLA} use a dimensionality reduction technique followed by a statistical two-sample test to compare the two distributions. \cite{Kulinski2020FeatureSD} develop conditional distribution hypothesis tests and propose a score-based test statistic for localizing distribution shift. In robotics, \cite{FaridVeerEtAl2022} apply a two-sample procedure to detect when a robot is operating under shifted conditions that harm its performance. However, these methods are typically designed for an offline (batch) setting, and there is no obvious way to use these methods online without either losing the guarantee, or being very inefficient statistically.

Another approach for detecting distribution shift, introduced by \cite{Vovk2003TestingEO}, uses conformal martingales to test for exchangeability and is currently the only technique for detecting distribution shift online (\cite{Vovk2021TestingRO, Vovk2021RetrainON, Eliades2020AHB, Volkhonskiy2017InductiveCM, Fedorova2012PluginMF, Ho2005AMF, podkopaev2021tracking, hu2020distributionfree}). These methods use conformal prediction to obtain p-values for each sample at test time, and then use these p-values to define a martingale. If the martingale grows large, then there has likely been a distribution shift.
\cite{Vovk2021RetrainON} is the most recent and most relevant to our work, as it combines conformal prediction with martingale theory to obtain an online distribution shift detector with a guarantee limiting the false positive rate. This work demonstrates good efficiency, i.e., detects distribution shifts quickly, on the Wine Quality dataset, which contains 11-dimensional feature vectors. 

However, these martingales generally do not perform well on more complex or higher-dimensional robotics settings (e.g. with image data), and they are not directly optimized to solve the problem of detecting distribution shift in an end-to-end manner. 
Additionally, these methods will only detect a distribution shift if the shift affects the specific predictor used to define the nonconformity score, which may be undesirable if there are other metrics that are also important, or if the overall predictor performance stays the same but the predictor now fails more often in more critical situations. Instead, we design a more efficient martingale based on a learned classifier; our martingale detects distribution shifts more quickly and does not have these drawbacks.

\section{Background}
\label{sec:background}

A martingale is a stochastic process (a sequence of random variables) where the conditional expectation of the next value, given all previous values, is the same as the most recent value.

\vspace{-2mm}
\begin{definition}[Martingale]
\label{def:martingale} 
A martingale is a sequence of random variables $R_1$, $R_2$, \dots, such that $E[|R_{n}|] < \infty$ and
$ E[R_{n+1} | R_1, \dots, R_n] = R_n $ for all $n$.
\end{definition}

Many stochastic processes of interest are martingales, and therefore there is a well-developed body of statistical theory on martingales that we can draw from~\cite{doob1971martingale,hall2014martingale,vovk2005algorithmic,shafer2019game}. Doob's martingale inequality~\cite{williams1991probability} formalizes the notion that the probability that a martingale grows very large is very low.

\vspace{-2mm}
\begin{prop}[Doob's Inequality]
\label{prop:doob}
For a martingale $R_n$ indexed by an interval $[0, N]$, and for any positive real number $C$, it holds that
\vspace{-1mm}
\begin{equation}
    \mathrm{Pr}\left[\sup_{0 \leq n \leq N} R_n \geq C\right] \leq \frac{E[\max(R_N,0)]}{C}.
\end{equation}
\end{prop}

In our work, we define a stochastic process $M_n$ based on the outputs of a trained predictive model. $M_n$ is a martingale if new data points observed at test time are \emph{exchangeable} with data points seen during training, and we can apply Doob's Inequality to bound the false positive rate of alerts that are issued. 
\vspace{-2mm}
\begin{definition}[Exchangeability]
\label{def:exchangeability} 
    A sequence of data points $X_1, X_2, \cdots, X_N$ is exchangeable if the probability of observing any permutation of $X_1, X_2, \cdots, X_N$ is equally likely.
\end{definition}

Under the hypothesis of exchangeability, the probability of $M_n$ growing large is small. In other words, if there is no distribution shift (the data points observed during training and after deployment are exchangeable), then the probability that our system falsely issues a warning ($M_n$ grows large) is small. Conversely, if the martingale grows large, then the data was likely not exchangeable, implying that a distribution shift occurred.

\section{Detecting Distribution Shift} 
\label{sec:method}

We propose a method for detecting distribution shift online in episodic robotics settings. Our method combines a learned, end-to-end approach with statistical martingale theory to issue alerts about distribution shifts quickly and with a guaranteed false positive rate.

\subsection{Problem Setup}

Let $D_{\mathrm{orig}} = (X_1, X_2, \cdots, X_n)$ be a sequence of past data points, where each point represents an episode of the robot executing in some environment, and let $D_{\mathrm{new}} = (X_1', X_2', \cdots ) $ be a sequence of new data points observed at test time. Formally, we aim to design a series of test functions
\begin{align} 
\psi_j: D_{\mathrm{orig}}, X_1', \cdots, X_j' \mapsto \lbrace T, F \rbrace \quad \forall j = 1,\  2, \cdots,
\end{align}
where the output T(rue) indicates that we have found a distribution shift (i.e., the $X_j'$ are not drawn from the same distribution as the original points $X_i$), and F(alse) indicates that we have not. 

We say that the test is $\epsilon$-sound if whenever there is no distribution shift, i.e., when the test data $(X_1', X_2', \cdots )$ are indeed exchangeable with $ D_{\mathrm{orig}}$, then 
$
    \Pr\big[\exists j, \psi_j (D_{\mathrm{orig}}, X_1', \cdots, X_j' ) = T \big] \leq \epsilon,
$
and this guarantee should hold for any distributions of $ D_{\mathrm{orig}}$ and test data  $(X_1', X_2', \cdots )$.
Intuitively, a test is $\epsilon$-sound if whenever there is no shift, a warning is never issued with high $(1-\epsilon )$ probability. 

Conversely, when there is a distribution shift, we want the test to issue a warning as soon as possible; i.e., we want a small $j$ such that $\psi_j$ outputs T(rue). Formally, we define the initial discovery time as the smallest $j$ such that $\psi_j$ is T(rue). 
While we will show that it is possible to guarantee soundness for \textit{any} data distributions, it is generally impossible to guarantee the initial discovery time (unless the test trivially issues a warning all the time). For example, in the case where the distribution shift is tiny, e.g., the total variation distance between $(X_1',\ X_2', \cdots)$ and the initial data $D_{\mathrm{orig}}$ is very small, there are fundamental lower bounds on how well a test can distinguish the two distributions~\cite{yu1997assouad}. In this paper, we devise a test that is guaranteed $\epsilon$-sound, and has low initial discovery time empirically.

\subsection{Proposed Method}

\textbf{Overview:} The key idea behind our method is that a predictor trained to distinguish between two samples, one of which is taken from $D_{\mathrm{new}}$ and the other of which is taken from $D_{\mathrm{orig}}$, can do no better than random chance if there has been no distribution shift.
That is, an indicator variable $Y_k$ that takes the value of 1 when the prediction model correctly predicts which sample originated from $D_{\mathrm{new}}$ and 0 otherwise is a Bernoulli random variable with parameter $p:= \mathrm{Pr}[Y_k = 1] = 0.5$ when no distribution shift has occurred. This is true no matter what the prediction model is, or how it was trained. 
We concretize this notion in Lemma~\ref{lemma:bernoulli} below. 

More formally, let $\Xc$ denote the sample space of both $D_{\mathrm{orig}}$ and $D_{\mathrm{new}}$, and let $X_i \in \Xc$ and $X_j' \in \Xc$ represent samples from $D_{\mathrm{orig}}$ and $D_{\mathrm{new}}$ respectively. We consider a trained neural network model $f:\Xc^2 \to \{0,1\}$ that takes as input a set of unordered, unlabeled samples $\{X_i,\ X_j'\}$, and predicts which of the two input samples is the more recent one (i.e. which is from $D_{\mathrm{new}}$). Note that $f$ is a binary classifier. At each time step $k$, we can then define an indicator variable $Y_k$ as follows:
\begin{equation}
\label{eq:indicator}
    Y_k = 
    \begin{cases}
        1\hspace{0.5cm} \text{if } f \text{ predicts correctly} \\
        0\hspace{0.5cm} \text{otherwise}.
    \end{cases}
\end{equation}
$Y_k$ is a Bernoulli random variable with $p = 0.5$ if no distribution shift has occurred (see Lemma~\ref{lemma:bernoulli}). 

\vspace{-2mm}
\begin{lemma}
\label{lemma:bernoulli}
Let $(X_1, X_2, \dots)$ be a sequence of data points with $X_i \in \Xc$ and let $f: \Xc^2 \to \{0,1\}$ be a model that predicts which input in an unordered pair of data points $\{X, X'\}$ was more recent. Define the indicator random variable $Y \in \{0,1\}$ as in Equation \eqref{eq:indicator}, so that $Y = 1$ whenever $f$ correctly predicts which input from $X$ and $X'$ was more recent. If the sequence $\{X_1, X_2, \dots\}$ is exchangeable, then it holds that $\mathrm{Pr}(Y = 1) = \frac{1}{2}$, regardless of the choice of classifier $f$.
\end{lemma}
\vspace{-1mm}
\begin{proof}
    See Appendix~\ref{appendix:proofs}.
\end{proof}
\vspace{1mm}

We can then use these $Y_k$ values to define a stochastic process $M_n$, which is a martingale under the hypothesis that there is no distribution shift --- in other words, if $Y_k$ is indeed a Bernoulli random variable with $p = 0.5$, then $M_n$ is a martingale. We update the value of $M_n$ after each new episode/prediction, and by Doob's Inequality, the probability that $M_n$ grows large is very small. 
If $M_n$ does grow large, then the assumption that the $Y_k$'s are Bernoulli random variables with $p = 0.5$ has most likely been violated, indicating that the samples from $D_{\mathrm{orig}}$ and $D_{\mathrm{new}}$ are not exchangeable and that a distribution shift has occurred with high probability. 

We can select a threshold $C$ such that if $M_n \geq C$, our method will issue an alert that the distribution has shifted. Doob's Inequality guarantees a false positive rate inversely proportional to $C$ as the probability that $M_n \geq C$ when there has been no distribution shift is upper bounded by $E[\max(M_N,0)]/C$.

\textbf{Choice of Martingale:}
In theory, any martingale $M_n$ constructed from Bernoulli ($p = 0.5$) random variables and computed with the $Y_k$ values defined in Equation~\ref{eq:indicator} would allow us to detect distribution shift, in the sense that $M_n$ would eventually grow large if the $Y_k$'s are not Bernoulli ($p = 0.5$). However, one desirable property for $M_n$ is that it should grow quickly if there has been a distribution shift (i.e. if the $Y_k$'s are not Bernoulli with $p = 0.5$).
Thus, we use an exponential martingale defined as follows:
\vspace{-1mm}
\begin{equation}
\label{eq:martingale}
    M_n = (e^{t \cdot S_n})/((q + p e^t)^n),
\end{equation}
where $S_n = \sum_{i=1}^{n} Y_k$, $p = q = 0.5$, and we use $t = 1$.
We prove in Lemma~\ref{lemma:martingale} that $M_n$ is indeed a martingale.

\vspace{-2mm}
\begin{lemma}
\label{lemma:martingale}
    Let $(Y_1, Y_2, \dots)$ be a sequence of exchangeable and identically distributed Bernoulli random variables with $\mathrm{Pr}\big[Y_i = 1\big] = p$, and define $S_n := \sum_{i=1}^n Y_i$. Then the stochastic process $\{M_n\}_{n=1}^\infty$, with $M_n = (e^{t \cdot S_n})/(((1 - p)+ p e^t)^n)$, is a martingale.
\end{lemma} 
\begin{proof}
    See Appendix~\ref{appendix:proofs}.
\end{proof}
\vspace{1mm}

Since $M_0 = 1$ and the martingale is non-negative, Doob's Inequality simplifies in this case to 
\vspace{-1mm}
\begin{equation*}
    \mathrm{Pr}\left[\sup_{0 \leq n \leq N} M_n \geq C\right] \leq \frac{1}{C},
\end{equation*}
by the law of total expectation (since $E[M_n] = E[E[M_n | M_{n-1}]] = E[M_{n-1}] \cdots = E[M_0]$). For our experiments, we use a threshold of $C = 100$, which guarantees a false positive rate of $\le 0.01$. That is, the threshold of $C = 100$ guarantees that we raise a false alarm with probability at most $1\%$.  

\textbf{Training Procedure:}
We now describe the procedure that we use to train $f$, which is used to compute the test functions $\psi_j$.
At train time, we observe a sequence of data points $D_{\mathrm{orig}}$, and divide these into three non-overlapping sets: (1) a randomly sampled held back set of ``unseen'' data points, which will not be used until test time, (2) a set of ``older'' data points from earlier in the sequence, and (3) a set of ``more recent'' data points from later in the sequence.
We then take pairs of randomly selected samples, one from the set of ``older'' data points and one from the set of ``more recent'' data points, and train a neural network to distinguish between the two. The input to this neural network model is a pair of randomly selected, shuffled samples, and the output is either 0 or 1, depending on which sample the model predicts to be from the set of ``more recent'' data points. This method is self-supervised --- it depends only on the ordering of the two samples (before shuffling), which can be labeled automatically.

At test time, we observe a sequence of data points $D_{\mathrm{new}}$. Each incoming data point $X_j'$ is paired with a randomly selected data point $X_i$ from the held back set of unseen data points from $D_{\mathrm{orig}}$. This pair of samples is then input into the trained model $f$, which makes a prediction; the output is used to update $Y_j$ (Equation~\ref{eq:indicator}) and $M_j$ (Equation~\ref{eq:martingale}). The test function $\psi_j := \mathbbm{1}\{M_j > C\}$, which issues an alert the first time the martingale is greater than $C$, will have a guaranteed false positive rate of $1/C$.

After a prediction is made for a data point $X_j'$ in $D_{\mathrm{new}}$ and $M_j$ is updated, $X_j'$ is added to the set of more recent data points from $D_{\mathrm{orig}}$. Then, the entire process (taking pairs of randomly selected samples, training $f$, making a prediction, and updating $M_j$) is repeated for $X_{j+1}'$. This retraining is necessary for detecting distribution shifts that occur during test time and have never been previously encountered during training.
Notably, the models that we use are small and easy to train. Additionally, we take a continual/incremental learning approach and constantly fine-tune our model when new data arrives, rather than retraining from scratch with each data point, further reducing the computational cost.

\subsection{Discussion} 

By design, our proposed method incorporates a number of strengths: (1) guaranteed
$\epsilon$-soundness, (2) self-supervised training, (3) low initial discovery time, and (4) the ability to accommodate high-dimensional input data (e.g., images) as the classifier $f$ may be a deep neural network with arbitrary architecture. By combining a deep learning approach with a statistical martingale approach, this method can be deployed online for episodic robotics settings.

Two limitations of our method are that (1) it applies only to episodic settings, and (2) to inform decision making, it is useful mostly for distributions that change gradually (relative to the number of deployed robots). The episodic setting is necessary to satisfy the exchangeability assumption for samples; data points from within the same episode or trajectory would potentially be highly correlated, and thus multiple unusual (i.e., unlikely) samples might not provide any more evidence for distribution shift than one unusual sample. To inform decision making, our proposed monitor has practical utility as a warning system if there is a time period when a distribution shift has occurred but catastrophe may yet be averted. 
Nonetheless, it is worth noting that even with a rapid environmental distribution shift, our method has the potential to prevent failure if a sufficient number of robots are deployed and collecting data simultaneously (since each robot can be considered a different episode), provided that the environmental change does not result in immediate failure.
Moreover, many distribution shifts may impact performance without causing a catastrophic failure (for example, a shift may only impact the ride quality for passengers of an autonomous aircraft by reducing tracking performance); it is still worthwhile to detect such shifts. Since our method is self-supervised, we can detect these shifts even when human operators cannot assess performance post-facto for each episode (like in a large fleet of deployed autonomous aircraft). 

Both of these limitations are ultimately properties of the problem setup of detecting distribution shift, and are not specific to our proposed methodology; understanding how to combine this warning system with other strategies for preserving the safety of learning-enabled systems (e.g., anomaly detection, data lifecycle analysis) represent promising avenues for future research \cite{SinhaSharmaEtAl2022}.

\section{Experiments}

We compare our method against two baselines. The first is the method described by Vovk et. al. in~\cite{Vovk2021RetrainON}, which we will refer to as the conformal martingale (CM) method. For this method, we use the nearest distance nonconformity score as recommended in~\cite{Vovk2021RetrainON} when no labels are available. For a second baseline, CM-FV, we slightly modify the CM method to use learned features from a pre-trained neural network; here, we use a nearest distance nonconformity score on lower dimensional feature vectors extracted from a pre-trained neural network model. In the next two subsections, we present simulator and hardware experiments for variants of visual servoing tasks; we also provide results on standard benchmarks from the distribution shift detection literature (CIFAR-100~\cite{CIFAR-100}, CIFAR-10~\cite{CIFAR-10}, and the Wine Quality dataset~\cite{WineQuality}) in Appendix~\ref{sec:experiments_synthetic}. For all methods, an alert is issued when the martingales reach a threshold of 100, in order to guarantee a false positive rate of $\leq 0.01$ (as explained in Section~\ref{sec:method}). 
For more experiment details, additional ablations, and videos of our robots, see Appendices~\ref{ap:x-plane} and~\ref{ap:free-flyer}, and the supplementary video.

\begin{figure}[b]
    \centering
    \includegraphics[width=0.7\linewidth]{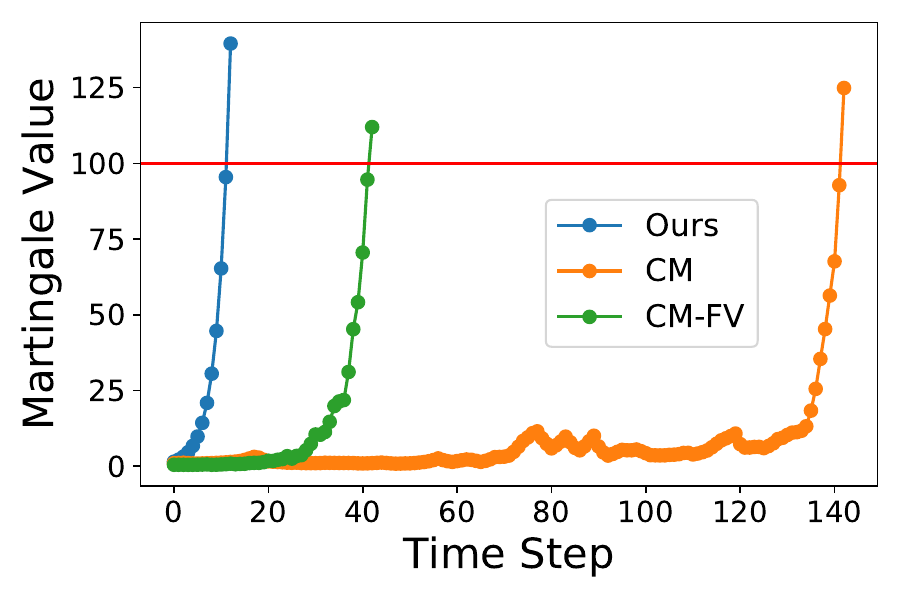}
    \caption{Martingale values for \textcolor{blue}{our method} (\textcolor{blue}{blue}), the \textcolor{orange}{CM method (\textcolor{orange}{orange})}, and the modified CM method \textcolor{green}{CM-FV} (\textcolor{green}{green}). The distribution shifts gradually over the course of the day, and an alert is issued when the martingales reach 100. In this example, our method issues an alert at time step 13, CM-FV issues an alert at time step 43, and CM issues an alert at time step 143.}
    \label{fig:xplane_time}
\end{figure}

\begin{figure*}[!htb]
    \centering
    \begin{subfigure}[c]{0.325\textwidth}
        \centering
        \includegraphics[width=\textwidth]{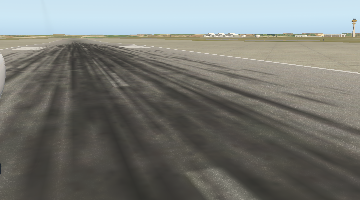}\vspace{1mm}
        \caption{Calibrated camera}
        \label{fig:camera_correct}
    \end{subfigure}
    \begin{subfigure}[c]{0.325\textwidth}
        \centering
        \includegraphics[width=\textwidth]{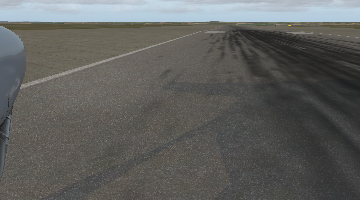}\vspace{1mm}
        \caption{Perturbed camera}
        \label{fig:camera_perturbed}
    \end{subfigure}
    \begin{subfigure}[c]{0.305\textwidth}
        \centering
        \includegraphics[width=\textwidth]{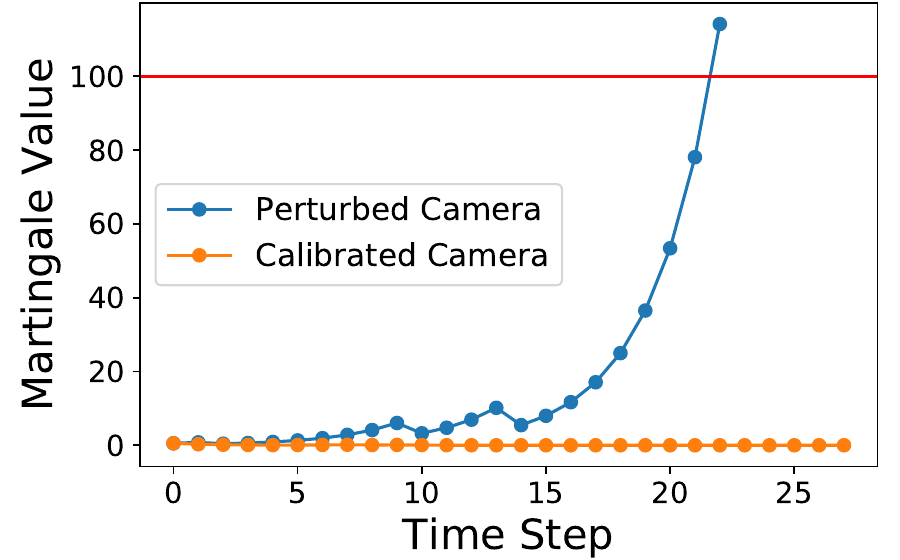}
        \caption{Martingale values}
        \label{fig:xplane_camera_results}
    \end{subfigure}
    \caption{Sample images generated with the X-Plane 11 flight simulator, with (\ref{fig:camera_correct}) a standard camera angle, and (\ref{fig:camera_perturbed}) a perturbed camera angle. (\ref{fig:xplane_camera_results}) shows martingale values for our method \textcolor{blue}{with} (\textcolor{blue}{blue}) and \textcolor{orange}{without} (\textcolor{orange}{orange}) distribution shift. With a distribution shift, the martingale grows rapidly, but without one, the martingale does not grow. We empirically observe both FNR = 0 and FPR = 0.}
    \label{fig:xplane_images_camera}
    \vspace{-3mm}
\end{figure*}

\subsection{X-Plane Simulator Experiments}
\label{sec:experiments_xplane}

We validate the performance of our method on image data from an autonomous aircraft that relies on an outboard camera in the photorealistic X-Plane 11 flight simulator. Each episode consists of the autonomous aircraft using a PID controller to taxi along the centerline of a runway. We pretrain a DNN to estimate the centerline distance from vision using only clear-sky, morning weather from the simulator ground truth. For monitor computation, each sample $X_i$ or $X_j'$ is one image sampled randomly from the episode. We test our method on two separate distribution shifts that are safety critical; we verify in Appendix \ref{ap:x-plane}~\cite{this-paper} that both shifts degrade the perception model and cause the autonomous aircraft to fail and run off the runway. 

We use our method to detect these realistic failure scenarios of learning-enabled robots and find that it significantly outperforms prior work, detecting gradual distribution shifts up to an order of magnitude faster than the baselines. We also show empirically that the guarantee on the false positive rate holds; i.e., when there is no distribution shift, our martingale does not grow.

\subsubsection{Gradual Daytime to Nighttime Shift} 
We first demonstrate that our method significantly outperforms both baselines and raises a warning before policy failure on the distribution shift of a simulated gradual daytime to nighttime lighting shift (see Figure~\ref{fig:system}).

\textbf{Dataset. } We use the X-Plane 11 flight simulator and NASA's XPlaneConnect Python API to create 1000 simulated video sequences taken from a camera attached to the outside of the plane as it taxis down the runway at different times throughout the day (with different weather conditions, starting positions, etc.)~\cite{xplane-simulator}. Each taxiing sequence consists of approximately 30 images of size 200x360x3. We randomly sample one image from each sequence. Fig.~\ref{fig:xplane_images} in Appendix~\ref{ap:x-plane}~\cite{this-paper} shows three example images.

\textbf{Experimental Setup. } We combine the morning and afternoon data points to form the training dataset. The evening data points are deployed in time order (i.e. the earliest evening images first) at test time. We train a basic neural network (with four convolutional layers followed by two linear layers) to predict which inputs are more recent, and run 100 trials of each experiment. (Refer to the Appendix for additional training details.) 

\textbf{Results. } Our method significantly outperforms both baselines. Our method issues an alert only 14.45 time steps into the evening data samples (on average over the 100 trials). With the CM method, the alert is issued after 161.18 time steps on average, and with the modified CM-FV method, the alert is issued after 37.44 time steps on average. Fig.~\ref{fig:xplane_time} shows an example plot of the growth of the martingale values for each method; an alert is issued after each martingale crosses the threshold of 100. The prompt alert from our method is particularly interesting because the early evening images (from just after 5:00pm) look visually very similar to those from earlier in the day. Notably, over 100 trials of the experiment, our method never fails to detect a distribution shift; i.e. we empirically observe no false negatives. The CM method fails to detect a distribution shift 34 times, and the CM-FV method fails to detect a distribution shift once. These numbers are summarized in Table~\ref{tab:xplane-day-night}.

These results indicate that our method performs well on realistic examples, and detects distribution shifts up to 11x more quickly than prior work. They also suggest that our method holds a larger efficiency advantage as the data increases in dimensionality (compare to results on synthetic datasets in the Appendix~\cite{this-paper}), and that both our end-to-end optimized methodology and our use of a learned model lead to a more rapid detection of distribution shifts.

Note that the computational cost of updating our model after each episode is very small, even on high-dimensional data. Each update in our experiments takes less than two seconds on a MacBook Pro M1 CPU.
Additional ablations with different hyperparameters and neural network architectures can be found in Appendix~\ref{ap:x-plane}~\cite{this-paper}; overall, the results do not vary significantly.

\begin{table} [tbp]
\vspace{2mm}
    \begin{center}
        \begin{tabular}{ l c c c}
            \toprule
            \multicolumn{4}{c}{\textbf{X-Plane, Day-to-Night Shift (100 Trials)}} \\
            & Ours & CM & CM-FV \\  
            \midrule 
            Mean Time Steps Until Alert	    & \textbf{14.45} & 161.18 & 37.44 \\
            False Negative Rate	        & \textbf{0} & 0.34 & 0.01 \\
            \bottomrule
        \end{tabular}
    \end{center}
    \caption{Results on the daytime-to-nighttime distribution shift with X-Plane data for our method, the CM method, and the CM-FV method. The first row summarizes the number of time steps before an alert is issued, and the second row summarizes the false negative rate for each method. All values are averaged over 100 trials; lower is better, best results shown in \textbf{bold}. Our method significantly outperforms the CM and CM-FV methods.}
    \label{tab:xplane-day-night}
\vspace{-3mm}
\end{table}

\begin{figure*}[!tb]
    \centering
    \begin{subfigure}[b]{0.235\textwidth}
        \centering
        \includegraphics[width=\textwidth]{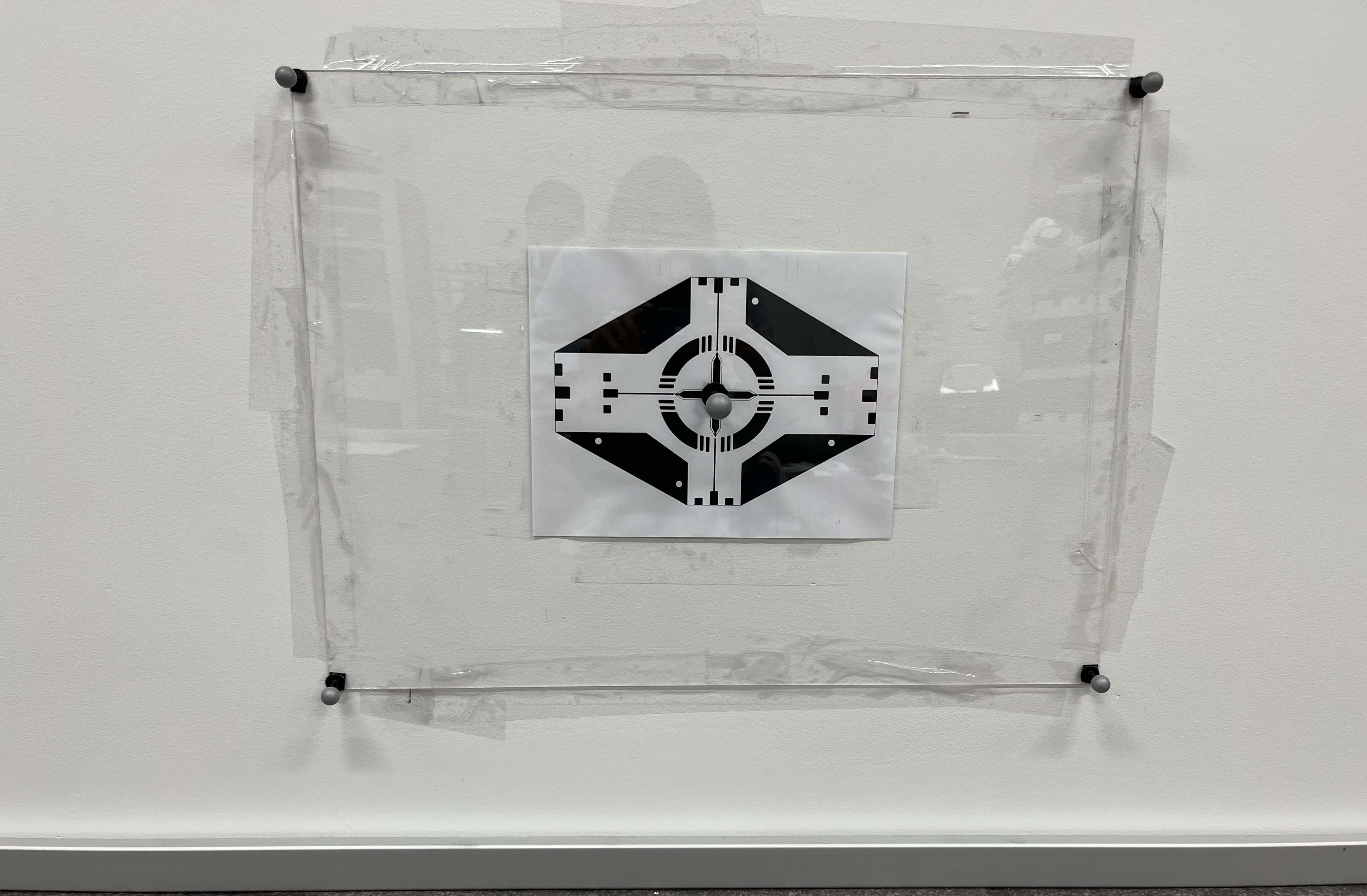}
        \caption{Initial condition}
        \label{fig:0episodes}
    \end{subfigure}
    \hfill
    \begin{subfigure}[b]{0.235\textwidth}
        \centering
        \includegraphics[width=\textwidth]{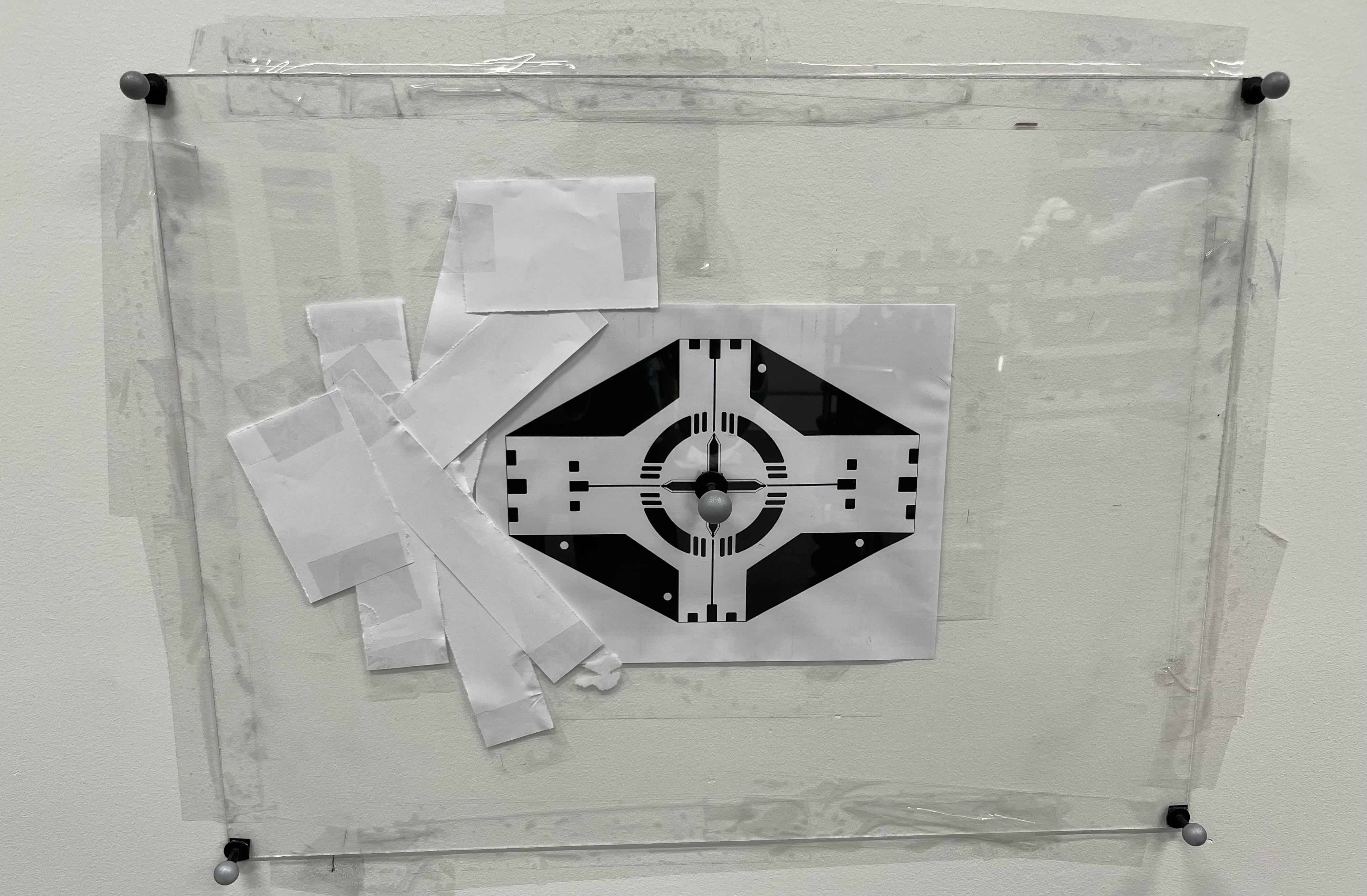}
        \caption{After 5 episodes}
        \label{fig:5episodes}
    \end{subfigure}
    \hfill
    \begin{subfigure}[b]{0.235\textwidth}
        \centering
        \includegraphics[width=\textwidth]{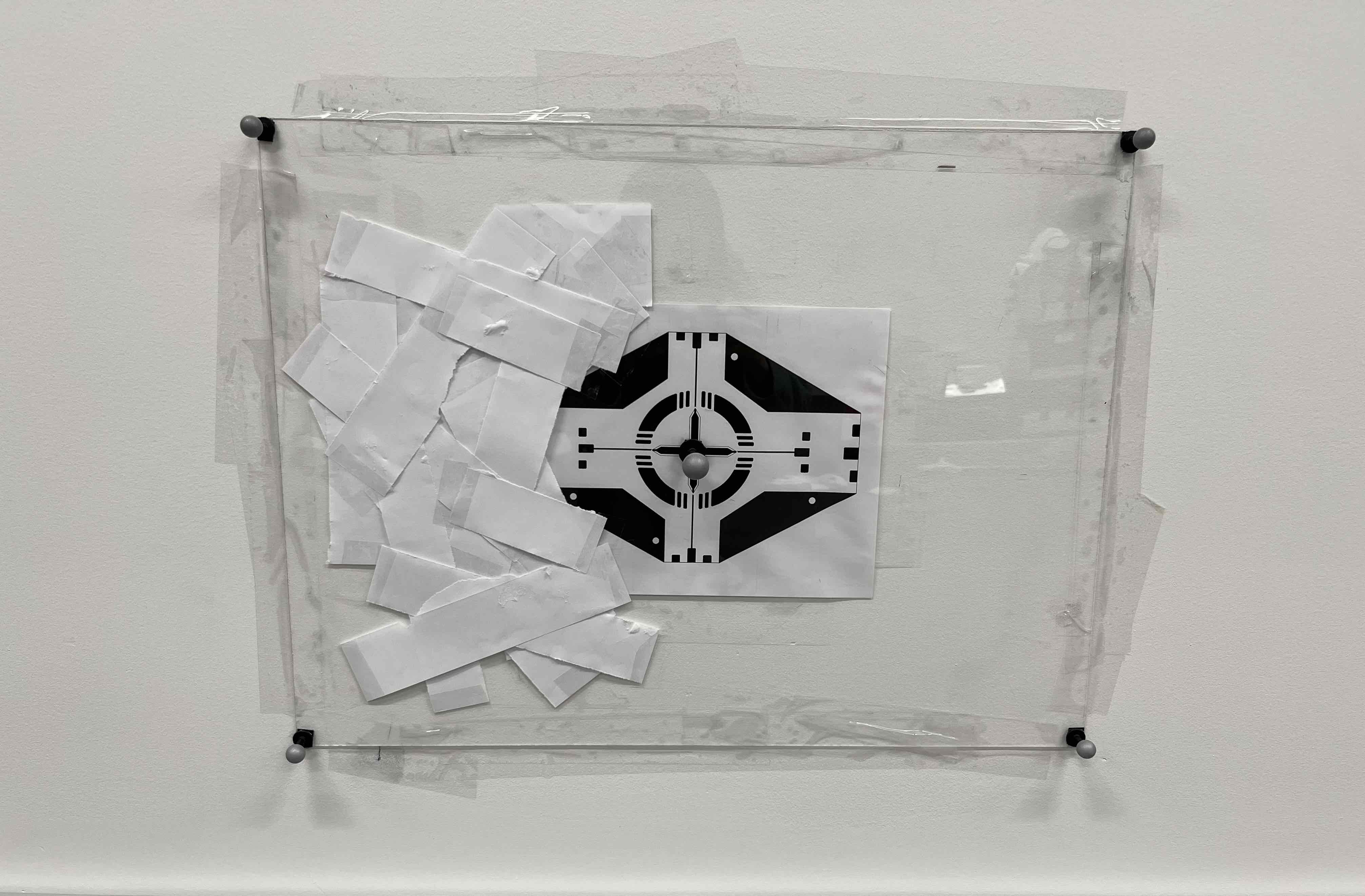}
        \caption{After 21 episodes}
        \label{fig:21episodes}
    \end{subfigure}
    \hfill
    \begin{subfigure}[b]{0.235\textwidth}
        \centering
        \includegraphics[width=\textwidth]{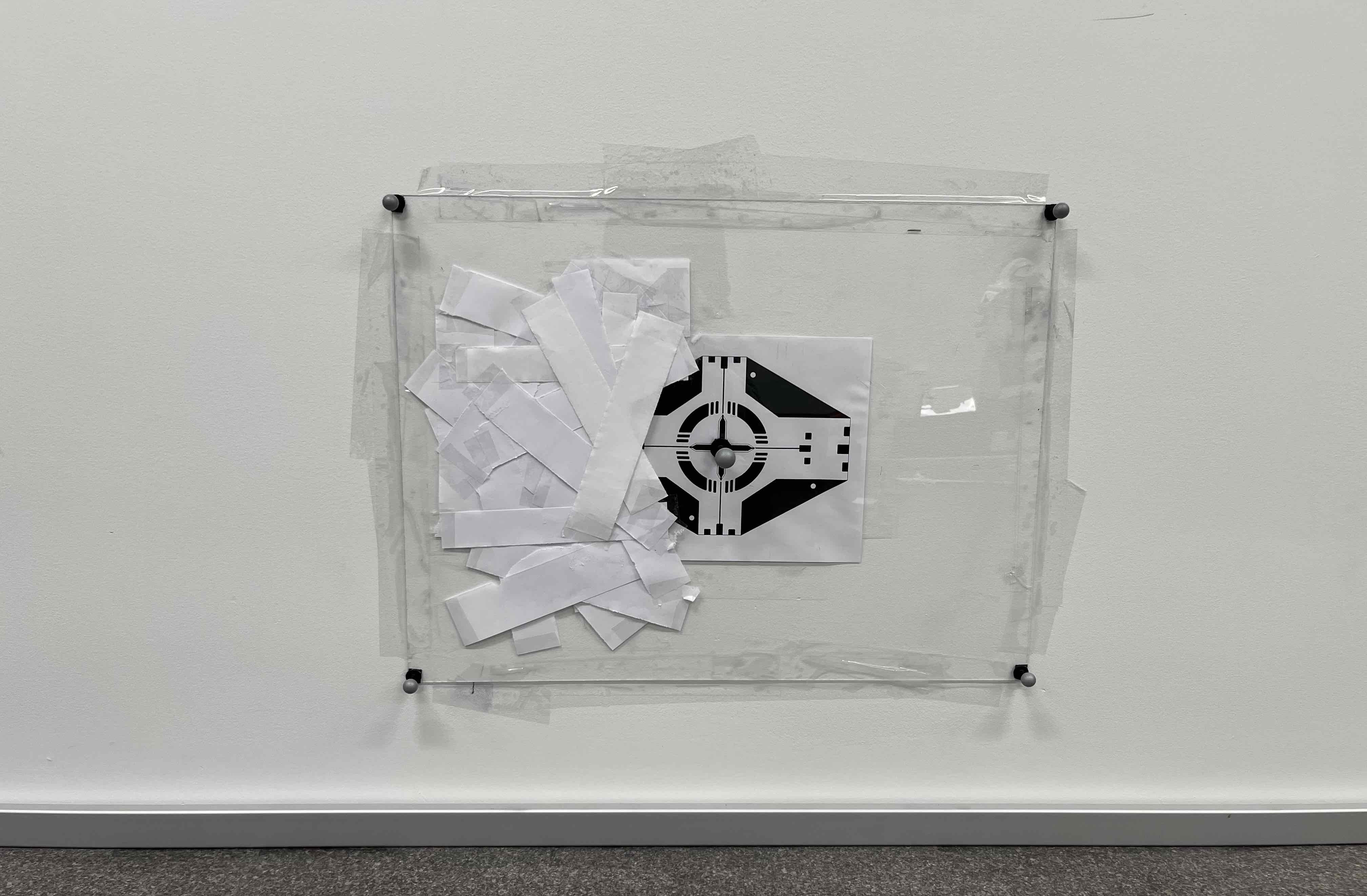}
        \caption{After 30 episodes}
        \label{fig:30episodes}
    \end{subfigure}
    \caption{Gradual degradation of the visual target, after (\ref{fig:0episodes}) 0 episodes, (\ref{fig:5episodes}) 5 episodes, (\ref{fig:21episodes}) 21 episodes, when our method issues an alert, and (\ref{fig:30episodes}) 30 episodes, when the robot fails to navigate to the target.  }
    \label{fig:freeflyer_degradation}
    \vspace{-1mm}
\end{figure*}

\subsubsection{Camera Angle Shift}
We consider a second set of simulations in which we compare the growth of our martingale with and without a distribution shift, specifically caused by a change in the camera angle (this could happen, for instance, if the camera was knocked slightly askew). Over the 100 trials of each scenario (distribution shift and no distribution shift), our method never fails to detect a distribution shift when there is indeed a change in camera angle (with an average detection time of 21.8 time steps), and never issues a false alert when there is no change in camera angle; i.e., we empirically observe no false negatives or false positives. 
Representative results for this experiment are shown in Fig.~\ref{fig:xplane_camera_results}, where ``Perturbed Camera'' is the distribution shift case and ``Calibrated Camera'' is the no distribution shift case. When there is no distribution shift, the martingale does not grow large; when there is a distribution shift, the martingale grows quickly.

\subsubsection{No Distribution Shift}
Finally, we demonstrate empirically that our false positive rate guarantee holds --- i.e., using a martingale threshold of $C = 100$, we have fewer than 1\% false alarms. In this set of experiments, there is no distribution shift between training and deployment. Over the 300 trials of this experiment with no shift, our method falsely issued an alert twice, emitting one alert at time step 57 and the other at time step 44. This represents a false positive rate of 0.0067, which is within our long-term theoretical limit of 0.01. A false positive can occur if the prediction model correctly predicts the more recent sample several times in a row by random chance. 
Because the false positive rate is very low, any alerts that are issued should be heeded. 
Similarly, the CM method issued a false alert twice (a false positive rate of 0.0067), and the CM-FV method issued a false alert once (a false positive rate of 0.0033). These results are unsurprising, since all three methods have a guarantee limiting the false positive rate over the long run. Our results are summarized in Table~\ref{tab:xplane-fpr}.

\begin{table} [htbp]
    \begin{center}
        \begin{tabular}{ l c c c}
            \toprule
            \multicolumn{4}{c}{\textbf{X-Plane, No Distribution Shift (300 Trials)}} \\
            & Ours & CM & CM-FV \\  
            \midrule 
            False Positive Rate	        & 0.0067 & 0.0067 & 0.0033 \\
            \bottomrule
        \end{tabular}
    \end{center}
    \caption{Empirical false positive rates on the X-Plane taxi scenario when there is no distribution shift. }
    \label{tab:xplane-fpr}
\end{table}

\subsection{Free-Flyer Hardware Experiments}
\label{sec:experiments_freeflyer}

We also perform hardware experiments using a free-flyer space robotics testbed with input from a forward-facing Intel Realsense D455 camera mounted on the side (see Fig.~\ref{fig:freeflyer_setup}). The free-flyer is a cold gas thruster-actuated 2D mobile robot that floats almost frictionlessly on a smooth granite table, developed to simulate zero-g or zero-friction conditions in aerospace robotics applications. In these experiments, we perform a learning-based visual servoing task that emulates autonomous spacecraft docking to demonstrate the efficacy of our distribution shift detection method. We use the visual pattern defined by the International Docking System Standard for spacecraft docking adapters as our main visual target~\cite{idss-standard}, as shown in Fig.~\ref{fig:freeflyer_target}. 

\vspace{1mm}
\begin{figure}
    \centering
    \begin{subfigure}[b]{0.23\textwidth}
        \centering
        \includegraphics[width=\textwidth]{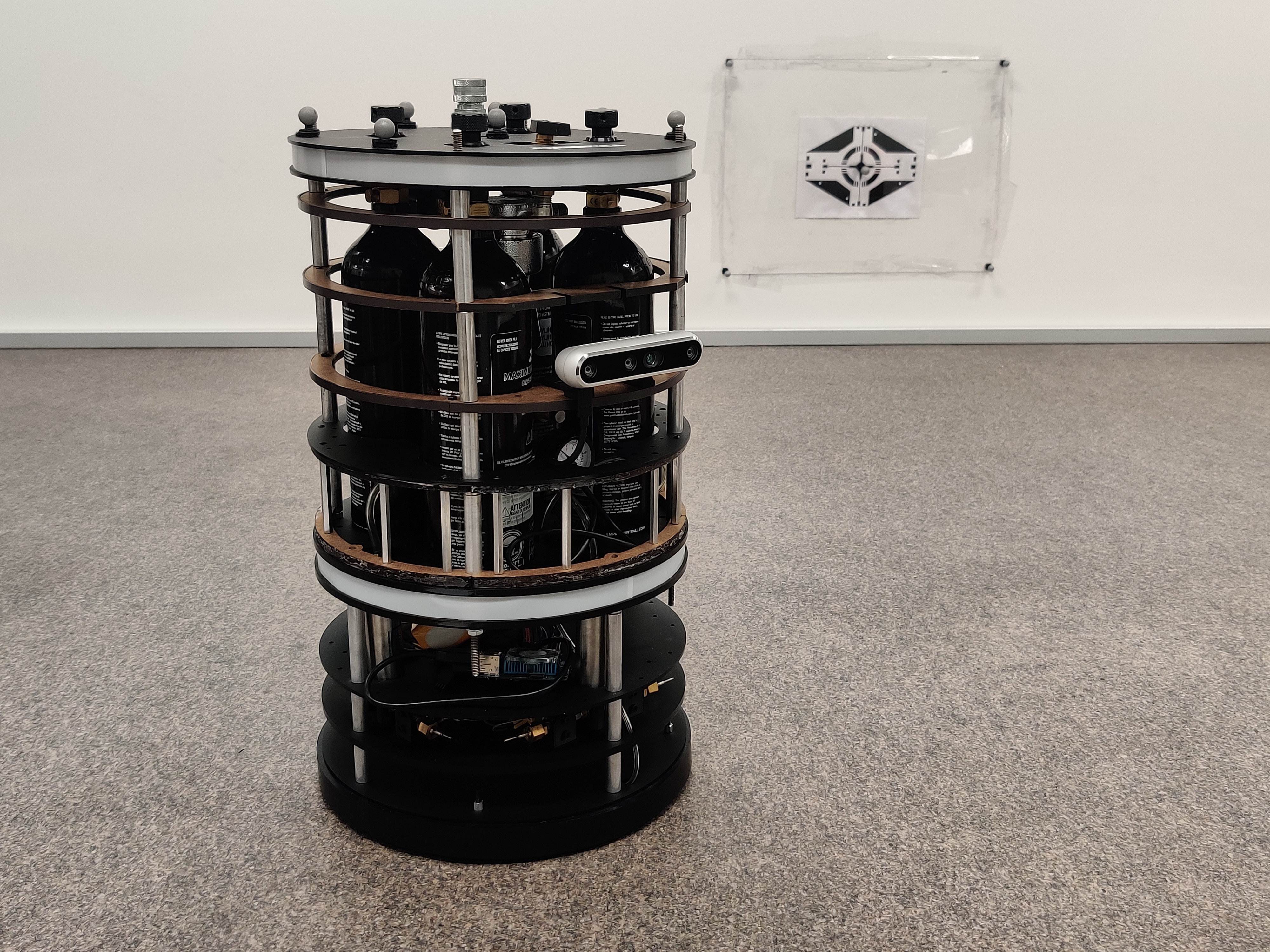}
        \caption{Free-flyer robot platform}
        \label{fig:freeflyer_setup}
    \end{subfigure}
    \hspace{0.5em}
    \begin{subfigure}[b]{0.23\textwidth}
        \centering
        \includegraphics[width=\textwidth]{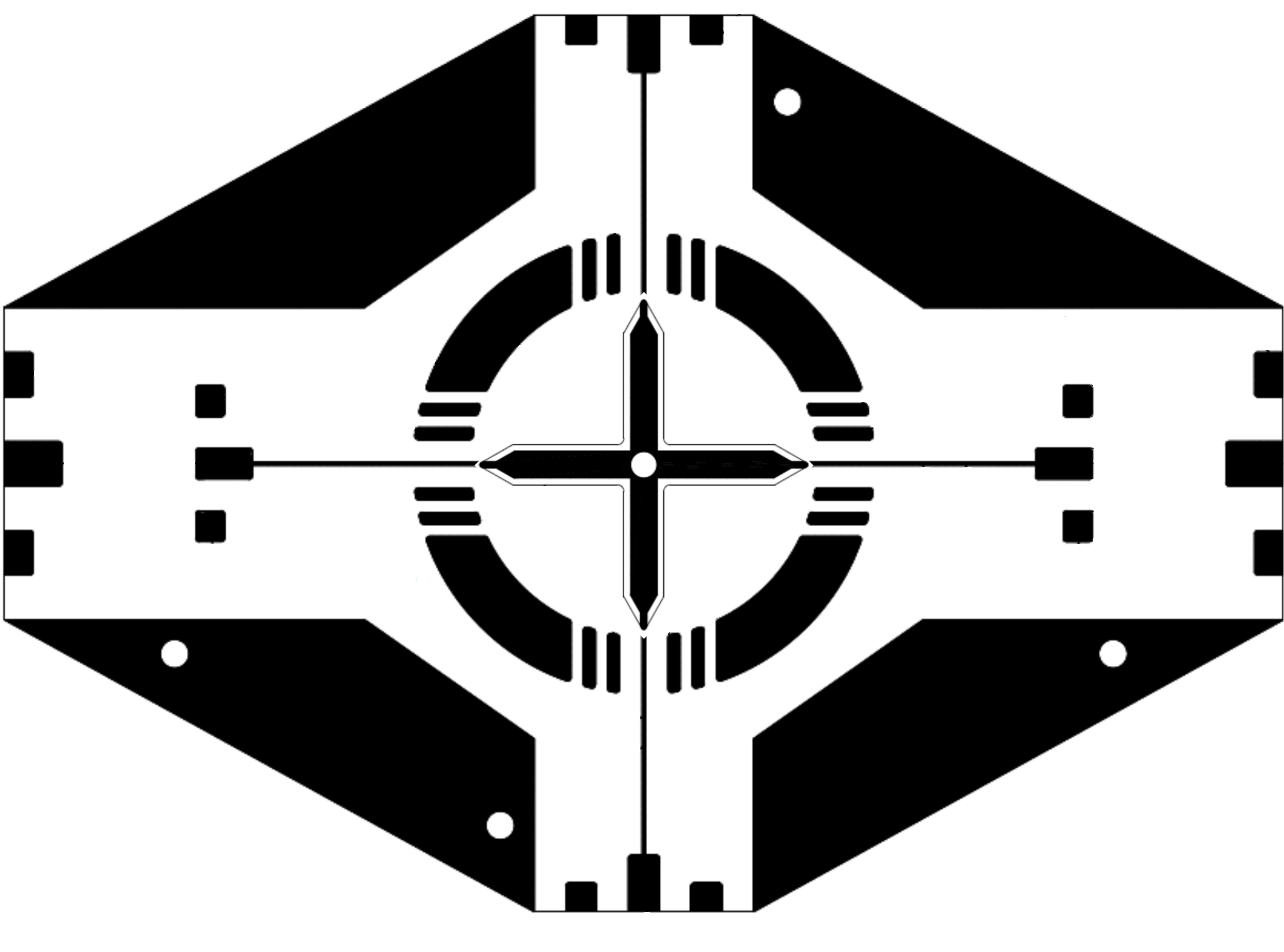}
        \caption{Visual servoing target}
        \label{fig:freeflyer_target}
    \end{subfigure}
    \caption{(\ref{fig:freeflyer_setup}) Hardware setup with camera mounted on the side of the mobile robot. (\ref{fig:freeflyer_target}) The visual servoing target that the robot navigates to. }
    \label{fig:freeflyer_platform}
\vspace{-3mm}
\end{figure}

\textbf{Experimental Setup:} First, we collect 10000 initial images (of size 360x640x3) and their associated ground truth relative pose with respect to the docking target using a motion capture system. We use this data to train a 6-layer CNN that directly predicts those offsets. During deployment, we pass the perception model outputs as measurements into a Kalman filter. We use the resulting visual state estimates to perform the visual servoing task of navigating the robot towards a docking position 50cm in front of the target with a simple PD controller. Each episode consists of a robot trajectory navigating from its initial position to the target. We consider an episode successful if the robot reaches a $\pm 10$ cm box centered around the goal position within 20 seconds. For the first 30 episodes (during training), 
no corruptions are added. Then, during deployment, we gradually add corruptions to the visual target until they cause the robot to fail. Specifically, after each episode, we add one additional white strip over the target (see Fig.~\ref{fig:freeflyer_degradation} for examples). This setup imitates repetitive autonomous spacecraft docking at the space station, where each additional docking could chip some paint off the docking target, slowly changing the target's visual appearance over time and leading to eventual spacecraft failure.

\textbf{Results: } Our method issues an alert at episode 21. The CM and CM-FV methods, in contrast, do not issue an alert before failure occurs at episode 30. Fig.~\ref{fig:freeflyer_degradation} shows the level of corruption at episode 21, when our method issues an alert, and at episode 30, when the robot fails to complete its task and Fig.~\ref{fig:freeflyer_martingale} plots the martingale values for each method. These results show that our method more rapidly warns us of distributional drift, allowing us to catch a realistic failure mode well \textit{before} failure occurs in a safety critical context.

\begin{figure}
\vspace{-2mm}
    \centering
    \includegraphics[width=0.7\linewidth]{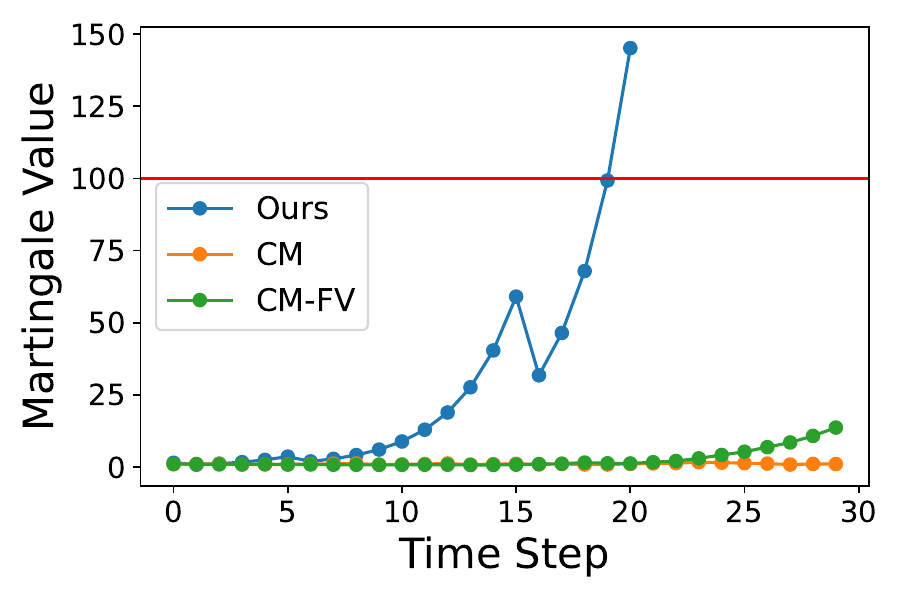}
    \caption{Martingale values for \textcolor{blue}{our method} (\textcolor{blue}{blue}), the \textcolor{orange}{CM method (\textcolor{orange}{orange})}, and the \textcolor{green}{CM-FV method} (\textcolor{green}{green}). Our method issues an alert at time step 21, \textit{before} the free-flyer fails to navigate to the target at time step 30.}
    \label{fig:freeflyer_martingale}
\vspace{-3mm}
\end{figure}

\section{Conclusion} 
\label{sec:conclusion}

In this work, we introduce a method for detecting distribution shifts on high-dimensional data in a streaming fashion. Our method is practical for robotics applications, and we demonstrate empirically that it performs well on photorealistic simulations of a plane taxiing down a runway -- detecting distribution shifts up to 11x more quickly than prior work -- as well as on a free-flyer hardware platform. By design, our method incorporates a number of strengths: (1) guaranteed
$\epsilon$-soundness, (2) self-supervised training, (3) low initial discovery time, and (4) the ability to accommodate high-dimensional input data (e.g., images). 
In future work, we would like to explore methods for detecting distribution shifts when the data is correlated~\cite{Barber2022ConformalPB}, as well as methods for combining this warning system with other strategies to preserve the safety of learning-enabled systems~\cite{SinhaSharmaEtAl2022}.






{
\bibliographystyle{ieeetr}
\bibliography{references}
}

\newpage
\onecolumn
\selectfont
\section*{Appendix}

\setcounter{lemma}{0}
\section{Proofs}
\label{appendix:proofs}

\begin{lemma}
Let $(X_1, X_2, \dots)$ be a sequence of data points with $X_i \in \Xc$ and let $f: \Xc^2 \to \{0,1\}$ be a model that predicts which input in an unordered pair of data points $\{X, X'\}$ was more recent. Define the indicator random variable $Y \in \{0,1\}$ as in Equation \eqref{eq:indicator}, so that $Y = 1$ whenever $f$ correctly predicts which input from $X$ and $X'$ was more recent. If the sequence $\{X_1, X_2, \dots\}$ is exchangeable, then it holds that 
\begin{equation*}
    \mathrm{Pr}(Y = 1) = \frac{1}{2},
\end{equation*}
regardless of the choice of classifier $f$.
\end{lemma}
\begin{proof}
Suppose we are given any two observations $\{X, X'\} \in \Xc^2$ from the sequence $(X_1, X_2, \dots)$, where $X$ was recorded at timestep $t$ and $X'$ was recorded at $t'$ so that $t \neq t'$. Since the data series is exchangeable, the events $t < t'$ and $t > t'$ are equally likely, and $\mathrm{Pr}[t > t' \ | \ \{X, X'\}] = \frac{1}{2}$. 
Therefore,
\begin{align*}
    \mathrm{Pr}\big[Y = 1 | \{X, X'\} \big] &= \mathrm{Pr}\big[Y = 1 \ | \ t > t', \ \{X, X'\} \big] \cdot \mathrm{Pr} \big[t > t'\ | \ \{X,X'\} \big]  \\
    &\quad \quad \quad + \mathrm{Pr} \big[Y = 1 \ | \ t < t' , \{X,X'\} \big] \cdot \mathrm{Pr} \big[t < t' \ | \ \{X,X'\} \big]   && \text{(Total probability)}\\
    &= \frac{1}{2} \Big(\mathrm{Pr}\big[Y = 1 \ | \ t > t', \ \{X, X'\} \big] + \mathrm{Pr} \big[Y = 1 \ | \ t < t', \ \{X, X'\} \big] \Big)  && \text{($X, X'$ exchangeable)}\\
    &= \frac{1}{2}, && \text{($f$ is deterministic)}
\end{align*}
since $f$ is deterministic: its output will be the same given $\{X, X' \}$ regardless of whether $t < t'$ or $t > t'$, 
so its prediction will be correct with probability $1$ for exactly one of the cases $t < t'$ or $t > t'$ and with probability $0$ for the other. Finally, apply the Tower rule to see that 
\begin{align*}
    \mathrm{Pr}\big[Y = 1 \big] &= \mathrm{E}\big[Y \big] \\
    &= \mathrm{E}\Big[\mathrm{E}\big[Y | \{X,X'\}\big]\Big] \\
    &= \mathrm{E}\Big[\mathrm{Pr}\big[Y=1 | \{X,X'\}\big]\Big] \\
    &= \frac{1}{2}.
\end{align*}
\end{proof}

\vspace{3mm}

\begin{lemma}
    Let $(Y_1, Y_2, \dots)$ be a sequence of exchangeable and identically distributed Bernoulli random variables with $\mathrm{Pr}\big[Y_i = 1\big] = p$, and define $S_n := \sum_{i=1}^n Y_i$. Then the stochastic process $\{M_n\}_{n=1}^\infty$, with
    \begin{equation*}
        M_n = \frac{e^{t \cdot S_n}}{((1 - p)+ p e^t)^n}
    \end{equation*}
    is a Martingale.
\end{lemma} 
\begin{proof}
    By De-Finetti's representation theorem, any sequence of exchangeable random variables can be written as a mixture of i.i.d. random variables, i.e. there exists a random variable $Z \in [0, 1]$ such that $Y_1, Y_2, \cdots$ are i.i.d. conditioned on $Z$. Then we have 
    \begin{align*}
        \mathrm{E}[M_{n+1} | M_1, \dots, M_n] &= \mathrm{E}_Z \left[ \mathrm{E}\Big[\frac{e^{tY_{n+1}}}{(1-p) + pe^t} M_n \ | \ M_1, \dots, M_n, Z\Big]  \right] && \text{(Def. of $M_n$) and tower property}\\
        &=\mathrm{E}_Z \left[  \mathrm{E}\Big[\frac{e^{tY_{n+1}}}{(1-p)+ pe^t} \ | \ M_1, \dots, M_n, Z \Big]  \right] M_n  && \text{(Take out what is known)} \\
        &= \mathrm{E}_Z \left[ \mathrm{E}\Big[\frac{e^{tY_{n+1}}}{(1-p) + pe^t} \mid Z\Big] \right] M_n  && \text{($Y_{n+1}$ is indep. of $Y_1, \dots, Y_n$ cond. on $Z$)} \\
        &=  \mathrm{E}\Big[\frac{e^{tY_{n+1}}}{(1-p) + pe^t} \Big]  M_n  && \text{(Tower)} \\
        &= \frac{(1-p) + p e^t}{(1-p) + p e^t} M_n = M_n && \text{(Evaluate the expectation)}
    \end{align*}
\end{proof}


\newpage
\twocolumn

\section{Synthetic Experiments}
\label{sec:experiments_synthetic}

In addition to the X-Plane and Free-Flyer robot experiments discussed in the main text, we empirically validate our method on several synthetic datasets and observe that it consistently outperforms prior work. 

\textbf{Datasets. } We use the CIFAR-100~\cite{CIFAR-100}, CIFAR-10~\cite{CIFAR-10}, and Wine Quality~\cite{WineQuality} datasets, which are standard benchmarks for existing work on distribution shift (e.g. \cite{Vovk2021RetrainON} evaluated their method on the Wine Quality dataset). The CIFAR-100 and CIFAR-10 datasets consist of 32x32x3 color images divided into 100 and 10 classes, respectively. To simulate various distribution shifts on the CIFAR datasets, we use the CIFAR-100-C and CIFAR-10-C datasets~\cite{hendrycks2019Benchmarking}, which are perturbed versions of the original CIFAR test sets. Each -C dataset includes 15 perturbed versions of the corresponding CIFAR test set, with perturbations such as brightness, Gaussian noise, motion blur, and fog. The Wine Quality dataset comprises 11-dimensional feature vectors for 4898 white and 1599 red wines. 

\textbf{Methods. } We compare our method against two baselines. The first is the CM method as described in~\cite{Vovk2021RetrainON}, using a nearest distance nonconformity score. The second is our modified CM-FV method, which uses a nearest distance nonconformity score applied to a much lower-dimensional feature vector extracted from a pre-trained neural network. 
For all methods, an alert is issued when the martingales reach a threshold of 100, in order to guarantee a false positive rate of $\leq 0.01$ (as explained in Section~\ref{sec:method}).

\textbf{Training Details. } The model that we use in our method for recency prediction in the CIFAR experiments is a simple convolutional neural network, with three 2D convolutional layers followed by two linear layers with a ReLU activation function. We train this predictor model with a batch size of 32, a constant learning rate of 1e-4, a binary cross-entropy loss function, and an Adam optimizer. All training is done on either a single Nvidia GeForce GTX TITAN X GPU or on a CPU (Macbook Pro M1 chip), since the training is not computationally expensive. 

\textbf{Experimental Setup. } We evaluate our method for detecting distribution shift using a simple neural network trained to distinguish between older and more recent images. For the CIFAR experiments, only unperturbed images are used during training, and perturbed images from a perturbation in the corresponding CIFAR-C dataset are used at test time. 
For the Wine Quality experiment, white wines are used during training and red wines are used at test time. We run 100 trials of each experiment for every perturbation in the CIFAR-100-C and CIFAR-10-C datasets, as well as for the white to red wine distribution shift, and average over the results.

\begin{figure}[htb]
    \centering
    \includegraphics[width=0.9\linewidth]{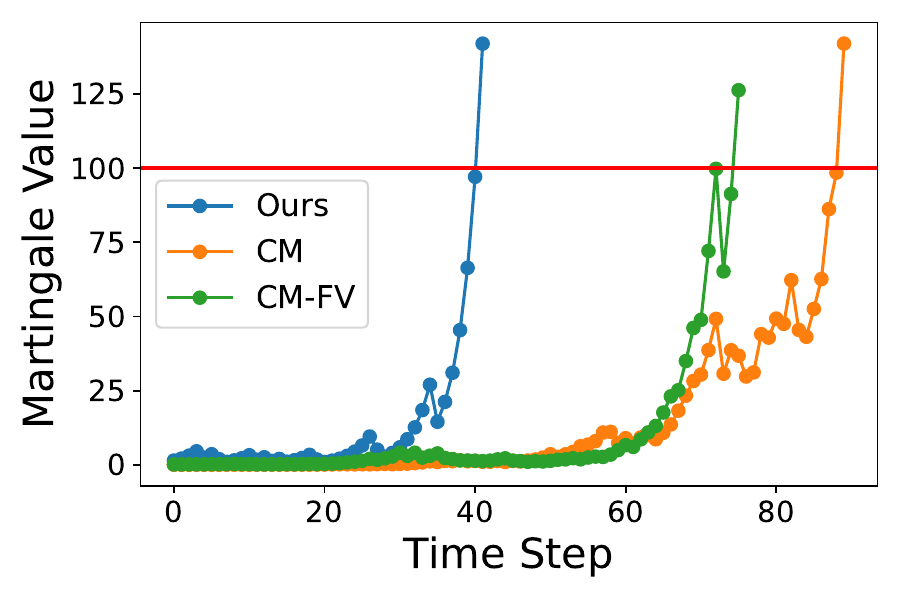}
    \caption{Martingale values for \textcolor{blue}{our method} (\textcolor{blue}{blue}), the \textcolor{orange}{CM method (\textcolor{orange}{orange})}, and the modified CM method \textcolor{green}{CM-FV} (\textcolor{green}{green}). An alert is issued when the martingales reach the threshold of 100. In this example, our method issues an alert at time step 42, the CM method issues an alert at time step 90, and the CM-FV method issues an alert at time step 76.}
    \label{fig:cifar10_motion_blur}
\end{figure}

\textbf{Results. } Our method consistently outperforms the CM and CM-FV methods. Tables~\ref{tab:cifar-100-time-steps} and~\ref{tab:cifar-10-time-steps} summarize our results for different perturbations of CIFAR-100 and CIFAR-10 over 100 trials. The numbers shown in the tables are the mean number of time steps needed for an alert and the false negative rate (in parentheses). The mean number of time steps is calculated over only the true positives (successful alerts); i.e. a failure to issue an alert is not included in the mean. If a method fails to issue an alert in over 95\% of the trials for some perturbation, we do not compute the mean and instead consider the method a failure for that perturbation. Note that this methodology artificially improves the computed means for the CM method, which sometimes fails to issue an alert (and would therefore otherwise include some very large numbers in the mean computation).

Out of the 15 perturbations of CIFAR-100, our method issues the quickest alert for 13 of the pertubations. Our method takes an average of 68.8 time steps after the distribution shift to issue an alert, while the CM method takes an average of 125.1 time steps, and the CM-FV method takes an average of 195.7 time steps. Note that for three of the perturbations, the CM method fails to detect the distribution shift at all; these perturbations are not counted in computing the overall average number of time steps taken by the CM method. Our method detects the distribution shift for all perturbations, and the CM-FV method detects the distribution shift at least 5\% of the time for all perturbations. 
Of the 15 perturbations of CIFAR-10, our method issues the quickest alert for 9 perturbations. Our method takes on average 59.4 time steps, while the CM method takes on average 128.0 time steps, and the CM-FV method takes on average 79.0 time steps. For one of the perturbations, the CM method fails to detect the distribution shift (and for two of the other perturbations, it fails to detect the distribution shift the majority of the time). Our method detects the distribution shift for all perturbations, and the CM-FV method detects the distribution shift the majority of the time for all perturbations. 
In Fig.~\ref{fig:cifar10_motion_blur}, we show an example plot for CIFAR-10 with the ``motion blur'' perturbation, demonstrating the martingale growth for our method, the CM method, and the CM-FV method. Our martingale grows more rapidly and issues an alert in fewer time steps. 

We note that our method generally issues an alert quickly across different perturbations, with a few exceptions for very difficult perturbations. For example, with the jpeg compression perturbation, it is very difficult to distinguish between the compressed and uncompressed images, particularly because the CIFAR images are only 32x32x3. In this case, our method may take longer to issue an alert; however, a perturbation of this magnitude may allow for more time before a problem occurs.

For the Wine Quality dataset, our method takes 16.4 time steps to issue an alert, while the CM method takes 22.1 time steps (on average over the 100 trials). (Since the Wine Quality dataset is already low-dimensional, we simply run the CM method as in~\cite{Vovk2021RetrainON}, rather than the CM-FV method.)

\def\hs{\hspace{-0.2cm}}
\begin{table}[t]
    \begin{center}
        \begin{tabular}{ l r l r l r l}
            \toprule
            \multicolumn{7}{c}{\textbf{Mean Time Steps Needed for Alert (100 Trials)}} \\
            \multicolumn{7}{c}{\textbf{CIFAR-100}} \\
            & \multicolumn{2}{c}{Ours} & \multicolumn{2}{c}{CM} & \multicolumn{2}{c}{CM-FV} \\  
            \midrule 
            Brightness	        & \textbf{21.0} & \hs(0)  & 272.8 & \hs(0.81)    & 391.6 & \hs(0.93) \\
            Contrast	        & \textbf{23.4} & \hs(0)  & 23.8  & \hs(0)       & 303.0 & \hs(0.53) \\
            Defocus Blur	    & \textbf{52.4} & \hs(0)  & 126.7 & \hs(0)       & 255.5  & \hs(0.38) \\
            Elastic Transform	& 168.7 & \hs(0)          & 264.2 & \hs(0.44)    & \textbf{139.9}  & \hs(0) \\
            Fog	                & \textbf{24.5} & \hs(0)  & 27.8  & \hs(0)       & 214.6  & \hs(0.06) \\
            Frost	            & \textbf{22.2} & \hs(0)  & 55.4  & \hs(0)       & 189.0  & \hs(0) \\
            Gaussian Noise	    & \textbf{43.5} & \hs(0)  & 86.2  & \hs(0)       & 107.1  & \hs(0) \\
            Glass Blur	        & \textbf{66.1} & \hs(0)  & --    & \hs(0.97)    & 97.3  & \hs(0) \\
            Impulse Noise	    & \textbf{32.8} & \hs(0)  & 36.7  & \hs(0)       & 98.5  & \hs(0) \\
            Jpeg Compression	& 312.6 & \hs(0)          & --    & \hs(0.99)    & \textbf{117.1}  & \hs(0) \\
            Motion Blur	        & \textbf{57.6} & \hs(0)  & 112.9  & \hs(0)      & 257.2  & \hs(0.23) \\
            Pixelate	        & \textbf{64.8} & \hs(0)  & --    & \hs(0.97)    & 196.3  & \hs(0.01) \\
            Shot Noise	        & \textbf{44.7} & \hs(0)  & 90.6  & \hs(0)       & 102.0  & \hs(0) \\
            Snow	            & \textbf{24.2} & \hs(0)  & 269.6 & \hs(0.48)    & 181.0  & \hs(0) \\
            Zoom Blur	        & \textbf{73.6} & \hs(0)  & 134.5 & \hs(0)       & 286.3  & \hs(0.04) \\
            \midrule
            Overall             & \textbf{68.8} & \hs(0)  & 125.1 & \hs(0.32)    & 195.7 & \hs(0.15) \\  
            \bottomrule
        \end{tabular}
    \end{center}
\caption{Mean number of time steps after a distribution shift occurs before an alert is issued, for our method, the CM method, and the CM-FV method (lower is better, best results shown in \textbf{bold}), under various perturbations of the CIFAR-100 dataset. The false negative rate for each method is shown in parentheses; mean time steps are computed over the true positive alerts. 100 trials are run for each experiment, and the mean number of time steps averages only over successful alerts. Note that the CM method fails to detect a distribution shift within 500 samples for three of the perturbations. Our method significantly outperforms both the CM and the CM-FV methods.}
\label{tab:cifar-100-time-steps}
\end{table}

\begin{table}[t]
    \begin{center}
        \begin{tabular}{ l r l r l r l}
            \toprule
            \multicolumn{7}{c}{\textbf{Mean Time Steps Needed for Alert (100 Trials)}} \\
            \multicolumn{7}{c}{\textbf{CIFAR-10}} \\
            & \multicolumn{2}{c}{Ours} & \multicolumn{2}{c}{CM} & \multicolumn{2}{c}{CM-FV} \\  
            \midrule 
            Brightness	        & \textbf{17.4} & \hs(0)    & 177.3 & \hs(0.11)    & 226.7 & \hs(0.22) \\
            Contrast	        & \textbf{19.6} & \hs(0)    & 23.2 & \hs(0)        & 224.4 & \hs(0.18) \\
            Defocus Blur	    & \textbf{41.4} & \hs(0)    & 100.3 & \hs(0)       & 111.2  & \hs(0) \\
            Elastic Transform	& 119.6 & \hs(0)            & 218.3 & \hs(0.18)    & \textbf{42.1}  & \hs(0) \\
            Fog	                & \textbf{22.3} & \hs(0)    & 28.1 & \hs(0)        & 89.5  & \hs(0) \\
            Frost	            & \textbf{20.3} & \hs(0)    & 64.5 & \hs(0)        & 52.9  & \hs(0) \\
            Gaussian Noise	    & 43.2 & \hs(0)             & 87.5 & \hs(0)        & \textbf{32.9}  & \hs(0) \\
            Glass Blur	        & 55.7 & \hs(0)             & 311.8 & \hs(0.79)    & \textbf{27.6}  & \hs(0) \\
            Impulse Noise	    & \textbf{29.4} & \hs(0)    & 38.5 & \hs(0)        & 32.6  & \hs(0) \\
            Jpeg Compression	& 315.8 & \hs(0)            & --  & \hs(1)         & \textbf{36.1}  & \hs(0) \\
            Motion Blur	        & \textbf{41.9} & \hs(0)    & 100.4 & \hs(0)       & 79.4  & \hs(0) \\
            Pixelate	        & 53.9 & \hs(0)             & 302.5 & \hs(0.85)    & \textbf{47.8}  & \hs(0) \\
            Shot Noise	        & 39.6          & \hs(0)    & 93.3 & \hs(0)        & \textbf{31.7}  & \hs(0) \\
            Snow	            & \textbf{20.1} & \hs(0)    & 120.1 & \hs(0)       & 54.2  & \hs(0) \\
            Zoom Blur	        & \textbf{50.5} & \hs(0)    & 126.4 & \hs(0)       & 95.4  & \hs(0) \\
            \midrule
            Overall             & \textbf{59.4} & \hs(0)    & 128.0 & \hs(0.20)     & 79.0 & \hs(0.03) \\  
            \bottomrule
        \end{tabular}
    \end{center}
\caption{Mean number of time steps after a distribution shift occurs before an alert is issued, for our method, the CM method, and the CM-FV method (lower is better, best results shown in \textbf{bold}), under various perturbations of the CIFAR-10 dataset. The false negative rate for each method is shown in parentheses; mean time steps are computed over the true positive alerts. 100 trials are run for each experiment, and the mean number of time steps averages only over successful alerts. Note that the CM method fails to ever detect a distribution shift within 500 samples for one of the perturbations. Our method significantly outperforms both the CM and the CM-FV methods.}
\label{tab:cifar-10-time-steps}
\end{table}


\section{X-Plane Simulator Experiments}\label{ap:x-plane}

\subsection{Daytime to Nighttime Shift: Additional Details } 
We use the X-Plane 11 flight simulator and NASA's XPlaneConnect Python API to create 1000 simulated video sequences taken from a camera attached to the outside of the plane as it taxis down the runway at different times throughout the day (with different weather conditions, starting positions, etc.)~\cite{xplane-simulator}. The first 295 sequences take place in the morning (8:00am-12:00pm), the next 344 sequences take place in the afternoon (12:00pm-5:00pm), and the last 361 sequences take place at night (5:00pm-10:00pm). Each taxiing sequence consists of approximately 30 images of size 200x360x3 (note that these images are much larger than those in the CIFAR dataset). 
We combine the morning and afternoon data points to form the training dataset with a total of 639 data points. These are then divided into 213 randomly sampled ``unseen'' images (to pair with the test time images), and 213 image pairs for training the neural network model.
Fig.~\ref{fig:xplane_images} shows example images from the morning, afternoon, and night. 

\begin{figure*}[!tb]
    \centering
    \begin{subfigure}[b]{0.325\textwidth}
        \centering
        \includegraphics[width=\textwidth]{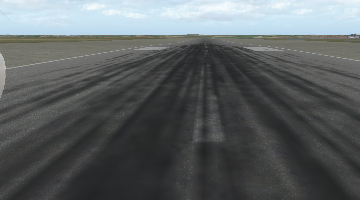}
        \caption{Morning}
        \label{fig:morning}
    \end{subfigure}
    \hfill
    \begin{subfigure}[b]{0.325\textwidth}
        \centering
        \includegraphics[width=\textwidth]{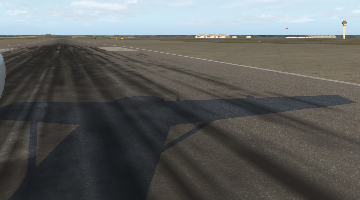}
        \caption{Afternoon}
        \label{fig:afternoon}
    \end{subfigure}
    \hfill
    \begin{subfigure}[b]{0.325\textwidth}
        \centering
        \includegraphics[width=\textwidth]{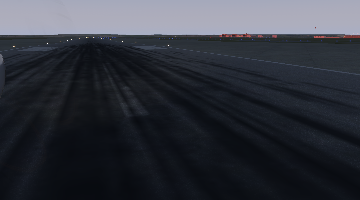}
        \caption{Night}
        \label{fig:night}
    \end{subfigure}
    \caption{Sample X-Plane 11 images with a distribution shift caused by gradually changing lighting conditions. }
    \label{fig:xplane_images}
    \vspace{-3mm}
\end{figure*}

\subsection{Camera Angle Shift: Additional Details}  
In this set of simulations, we compare the growth of our martingale with and without a distribution shift, where the distribution shift is a change in the camera angle. We again use the X-Plane 11 flight simulator to create 600 video sequences taken from a camera attached to the outside of the plane as it taxis down the runway~\cite{xplane-simulator}. These sequences occur at randomly initialized times between 8:00am and 10:00pm. 
Of these 600 sequences, 400 are taken with a standard camera angle, and 200 are taken with a slightly perturbed camera angle (see Fig.~\ref{fig:xplane_images_camera}), simulating the camera being knocked slightly askew. Each taxiing sequence consists of approximately 30 images of size 200x360x3, and we randomly sample one image from each sequence.
At training time, we observe 200 samples with the standard camera angle. At test time, we observe either 200 samples with the perturbed camera angle (a distribution shift), or 200 different samples with the standard camera angle (no distribution shift). This experiment can be thought of as simulating a calibrated camera setup in the standard case, and a camera that has been knocked slightly askew in the perturbed case.

\subsection{No Distribution Shift: Additional Details} 
We again use the X-Plane 11 flight simulator to create 1000 video sequences taken from a camera attached to the outside of the plane as it taxis down the runway~\cite{xplane-simulator}. All sequences occur in the morning (with randomized starting times, starting positions, and weather conditions), and there is no distribution shift. Each taxiing sequence consists of approximately 30 images of size 200x360x3, and we randomly sample one image from each sequence. 
At training time, we observe 200 samples from randomly sampled episodes in the generated dataset. At test time, we observe 200 samples from different randomly sampled episodes in the generated dataset. 

\subsection{Model and Training } 
The predictor model that we use in our method for recency prediction in the X-Plane experiments (with images of size 200x360x3) is a simple convolutional neural network, with four 2D convolutional layers followed by two linear layers with a ReLu activation function (see Figure~\ref{fig:cnn_diagram}). We train this predictor model with a batch size of 32, a constant learning rate of 1e-4, 
a binary cross-entropy loss function, and an Adam optimizer. All training is done on either a single Nvidia GeForce GTX TITAN X GPU or on a CPU (Macbook Pro M1 chip), since the training is not computationally expensive. 

Note that we simply update our model with every prediction as we obtain additional information (as opposed to retraining from scratch with each data point). Because our task is simply to discriminate between older and more recent samples, it is not a very complicated learning problem, and the update will necessarily be much less expensive than the downstream task being accomplished with the samples. 
Using the simple neural network architecture described above, each model update takes under two seconds on a Macbook Pro M1 laptop CPU. 
Thus, we found the associated computational cost very small even on high-dimensional data.

For all methods, an alert is issued when the martingales reach a threshold of 100, in order to guarantee a false positive rate of $\leq 0.01$ (as explained in Section~\ref{sec:method}). 

Since we want to benchmark the ability of our method and existing algorithms to detect shifts in anticipation of system failures, we compare these methods on vision data generated by running the PID controller with ground truth state information. This ensures that the monitors cannot simply detect a shift because the aircraft has failed (i.e. because the state distribution has changed so much that the images have hangars filling the frame rather than the runway); rather, they need to detect the environmental changes degrading the vision system.

\begin{figure}[h]
    \centering
    \includegraphics[width=0.25\linewidth]{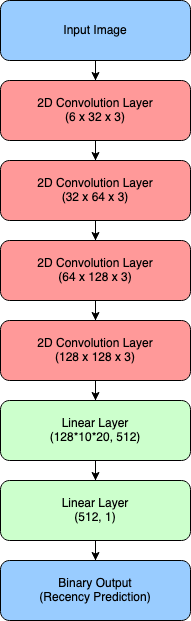}
    \caption{Model architecture for the predictor model used in our method for the X-Plane experiments. }
    \label{fig:cnn_diagram}
\end{figure}

\subsection{Additional Experimental Results} 
\textbf{Failure Caused by Distribution Shifts: }
We also run simulations on the X-Plane 11 flight simulator of an autonomous aircraft that uses a pre-trained neural network called TaxiNet \cite{xplanedataset}, along with a PID controller, to taxi down the runway from vision. The purpose of this experiment is to illustrate that the distributional shifts in \ref{sec:experiments_xplane}.A and \ref{sec:experiments_xplane}.B cause a deep-learned, vision-based controller to fail, and that in the setting in \ref{sec:experiments_xplane}.A, our method detects the shift before a failure occurs. To this end, the TaxiNet model has been trained to output the current cross-track error, that is, the distance to the center line of the runway, exclusively from input dashcam image recorded in the morning (8am-12pm). The PID controller then regulates the perceived distance to the center line. The labels to train TaxiNet were recorded from the simulator's ground-truth state. 

We first induce the day-night distribution shift in section \ref{sec:experiments_xplane}.A, by sampling test time episodes between 5:00pm and 10:00pm. We find that the plane consistently fails to taxi down the runway after about 6:20pm. Our method issues an alert on average at 5:12pm, the CM-FV method issues an alert on average at 5:30pm, and the CM method issues an alert on average at 7:25pm (see Table~\ref{tab:xplane_alert_time}). Thus, in this example, our method and the CM-FV method issue alerts before a failure occurs, while the CM method does not. Secondly, when we induce a shift by knocking the camera askew as in \ref{sec:experiments_xplane}.B, the TaxiNet control stack consistently fails.

\begin{table} [htbp]
	\begin{center}
		\begin{tabular}{c c c}
			\toprule
			\multicolumn{3}{c}{\textbf{Mean Alert Time, X-Plane Day-to-Night Shift}} \\
			\hspace{5mm} Ours \hspace{5mm} & \hspace{5mm} CM \hspace{5mm} &  CM-FV \\  
			\midrule 
            5:12pm  & 7:25pm    &  5:30pm\\
            \bottomrule
		\end{tabular}
	\end{center}
	\caption{Mean time at which an alert is issued for our method, the CM method, and the CM-FV method. Failure occurs at 6:20pm. Our method issues an alert 68 minutes before the failure, the CM-FV method issues an alert 50 minutes before the failure, and the CM method does not issue an alert until 65 minutes after the failure has occurred. }
	\label{tab:xplane_alert_time}
\end{table}

\textbf{Neural Network Architecture Ablations: }
We also ran several additional hyperparameter sweeps over the neural network architecture and other hyperparameters, and we found that overall, our results do not change significantly. We varied the learning rate from 1e-3 to 1e-5, the batch size from 8 to 128, and several parameters of the neural network architecture (including the number of convolutional layers, the size of the convolutional layers, and the stride). For the X-plane daytime to nighttime shift experiment, the average number of time steps taken by our method before an alert is issued ranges from 13.4 to 17.1. For the X-plane camera angle shift experiment, the average number of time steps taken by our method before an alert is issued ranges from 19.3 to 25.1. For the X-plane experiment with no distribution shift, our method never falsely issues an alert. All of these results are consistently faster than the other baseline methods in situations where an alert should be raised. Our results are summarized in Table~\ref{tab:nn_ablations}.

Note that because we do not constrain \textit{how} the distribution might change, it is impossible to make guarantees about how quickly any model will become sensitive to these distribution shifts. For example, an unconstrained adversary will be able to choose a shift that will evade detectability, regardless of the algorithm used (whether learning-based or not). However, such adversarial shifts are unlikely to occur in practice. Instead, we show in our work that by using the most powerful detection methods that currently exist (i.e., deep learning), we achieve better results than any alternative methods on the types of shifts that one would expect in practice.

\section{Free-Flyer Hardware Experiments}\label{ap:free-flyer}

To collect the initial image data, we randomly positioned the robot at different locations on the granite table while ensuring that the visual target stayed within the field of view of the camera. We captured 10000 images, along with the associated heading and position offsets, $[\Delta x, \Delta y \Delta \theta]$ using the Optitrack motion capture system. 

\newpage 
\onecolumn

\newcommand{\multrow}[1]{\begin{tabular}{@{}c@{}} #1 \end{tabular}}
\renewcommand{\arraystretch}{1.2}

\begin{table}[ht]
\centering
\begin{tabular}{|*{5}{c|}}
    \hline
    \multicolumn{2}{|c|}{}   & \textbf{Day-Night}  & \textbf{Camera Angle}  & \textbf{No Shift}  \\
    \hline
    \multicolumn{2}{|c|}{\textbf{Baseline:} \multrow{Batch size = 32 \\ Convergence = 1e-4 \\ Learning rate = 1e-4}}  & 13.9 & 21.8 & --- \\
    \hline 
    \multirow[c]{4}{*}{\textbf{Learning Rate}}           
            &   1e-3    &    15.5       &     21.2      &      ---     \\  
            &   5e-4    &    14.8       &     19.4      &      ---     \\  
            &   5e-5    &    14.2       &     25.1      &      ---     \\
            &   1e-5    &    15.8       &     21.7      &      ---     \\
    \hline
    \multirow[c]{4}{*}{\textbf{Batch Size}}           
            &   8       &    14.8       &     21.2      &      ---     \\  
            &   16      &    14.1       &     21.6      &      ---     \\  
            &   64      &    14.4       &     23.2      &      ---     \\  
            &   128     &    15.1       &     23.2      &      ---     \\  
    \hline
    \multirow[c]{5}{*}{\textbf{NN Architecture}}    
            & \multrow{Conv: \\ 3 x 16 x 3 \\ 16 x 32 x 3 \\ 32 x 64 x 3 \\ 64 x 128 x 3 \\ FC: \\ ... x 512 \\ 512 x 1} & 13.8 & 21.7 & --- \\   
            \cline{2-5}
            & \multrow{Conv: \\ 3 x 32 x 7 \\ 32 x 64 x 5 \\ 64 x 128 x 3 \\ 128 x 128 x 3 \\ FC: \\ ... x 512 \\ 512 x 1} & 13.4 & 20.3 & --- \\   
            \cline{2-5}
            & \multrow{Default configuration, \\ but replacing max-pool \\ with stride-2} & 14.4 & 19.3 & --- \\   
            \cline{2-5}
            & \multrow{Conv: \\ 3 x 32 x 3 (stride-2) \\ 32 x 32 x 3 \\ 32 x 64 x 3 (stride-2) \\ 64 x 64 x 3 \\ 64 x 128 x 3 (stride-2) \\ 128 x 128 x 3 \\ 128 x 128 x 3 (stride-2) \\ FC: \\ ... x 512 \\ 512 x 1} & 17.1 & 24.6 & --- \\   
            \cline{2-5}
            & \multrow{Conv: \\ 3 x 32 x 3 (stride-2) \\ 32 x 64 x 3 (stride-2) \\ 64 x 128 x 3 (stride-2) \\ FC: \\ ... x 512 \\ 512 x 1} & 14.1 & 21.1 & --- \\
    \hline
\end{tabular}
\caption{Results on the X-plane experiments with different hyperparameters and neural network architectures, averaged over 10 trials. Overall, the results do not vary significantly. Our method never raises an alert when there is no distribution shift.}
\label{tab:nn_ablations}
\end{table}

\end{document}